\documentclass[11pt]{article} 

\usepackage[utf8]{inputenc} 

\usepackage{geometry} 
\usepackage{graphicx}

\usepackage[parfill]{parskip}

\usepackage{hyperref}
\usepackage{natbib}
\usepackage{booktabs}
\usepackage{array} 
\usepackage{paralist} 
\usepackage{verbatim} 
\usepackage{subfig} 

\usepackage{amsmath,bm,bbm}
\usepackage{amssymb}
\usepackage{amsthm}
\begingroup
    \makeatletter
    \@for\theoremstyle:=definition,remark,plain\do{
        \expandafter\g@addto@macro\csname th@\theoremstyle\endcsname{
            \addtolength\thm@preskip\parskip
            }
        }
\endgroup

\usepackage[ruled,vlined]{algorithm2e}

\usepackage{fullpage}

\usepackage{color}

\usepackage{cleveref}

\newtheorem{remark}{Remark}[section]

\newtheorem{assumption}{Assumption}[section]
\newtheorem{lemma}{Lemma}

\newtheorem{definition}{Definition}[section]
\newtheorem{theorem}{Theorem}
\theoremstyle{remark}
\newtheorem{example}{Example}[section]

\newcommand{\calF}{\mathcal{F}}
\newcommand{\calX}{\mathcal{X}}
\newcommand{\calV}{\mathcal{V}}
\newcommand{\calR}{\mathcal{R}}
\newcommand{\calRup}{\bar{\calR}}

\newcommand{\bbR}{\mathbb{R}}

\newcommand{\bbE}{\mathbb{E}}
\newcommand{\bbP}{\mathbb{P}}

\newcommand{\vvta}[1]{\Vert #1 \Vert}
\newcommand{\lr}[1]{\left( #1 \right)}
\newcommand{\lrr}[1]{\left[ #1 \right]}
\newcommand{\lrrr}[1]{\left\{ #1 \right\}}
\newcommand{\tn}[1]{\textnormal{#1}}

\newcommand{\ie}{\textit{i}.\textit{e}., }
\newcommand{\eg}{\textit{e}.\textit{g}. }

\newcommand{\lspo}{\ell_{\textnormal{SPO}}}
\newcommand{\lspop}{\ell_{\textnormal{SPO+}}}
\newcommand{\rspo}{R_{\textnormal{SPO}}}

\newcommand{\vavg}{v_{\mathrm{avg}}}
\newcommand{\ravg}{r_{\mathrm{avg}}}

\usepackage{fancyhdr} 
\usepackage[nottoc,notlof,notlot]{tocbibind}
\usepackage[titles,subfigure]{tocloft}

\title{Online Contextual Decision-Making with a Smart Predict-then-Optimize Method}

\author{Heyuan Liu$^1$ \and Paul Grigas$^1$}

\date{
$^1$Department of Industrial Engineering and Operations Research \\
University of California, Berkeley\thanks{\texttt{\{heyuan\_liu, pgrigas\}@berkeley.edu}} \\[2ex]
\today
}

\begin{document}

\maketitle

\begin{abstract}
    We study an online contextual decision-making problem with resource constraints. At each time period, the decision-maker first predicts a reward vector and resource consumption matrix based on a given context vector and then solves a downstream optimization problem to make a decision. The final goal of the decision-maker is to maximize the summation of the reward and the utility from resource consumption, while satisfying the resource constraints. We propose an algorithm that mixes a prediction step based on the ``Smart Predict-then-Optimize (SPO)'' method with a dual update step based on mirror descent. We prove regret bounds and demonstrate that the overall convergence rate of our method depends on the $\mathcal{O}(T^{-1/2})$ convergence of online mirror descent as well as risk bounds of the surrogate loss function used to learn the prediction model. Our algorithm and regret bounds apply to a general convex feasible region for the resource constraints, including both hard and soft resource constraint cases, and they apply to a wide class of prediction models in contrast to the traditional settings of linear contextual models or finite policy spaces.  We also conduct numerical experiments to empirically demonstrate the strength of our proposed SPO-type methods, as compared to traditional prediction-error-only methods, on multi-dimensional knapsack and longest path instances. 
\end{abstract}

\section{Introduction}\label{sec:intro}

    Decision-making over time in the presence of uncertainty is a common task across many applications of machine learning. Some typical example problems include online network revenue management, resource allocation, and advertisement bidding. In these settings, there is a trade-off between immediate rewards and rewards received at a later time. This trade-off exists since each decision that is made consumes some of a limited amount of resources. 
    Often, the decision-maker does not have full knowledge of the relevant parameters dictating the amount of the rewards and resources consumed at time $t$, and instead has available contextual information that is related to these parameters and can be used to reduce uncertainty in the decision-making process. Indeed, contextual information such as search history, previous reviews, users characteristics, and many others may be available. For example, in online advertising we may not precisely know the probability that the user would click on a given advertisement, but we may build a machine learning model for predicting this probability based on characteristics of the user and the advertisement.
    
    Recently, there has been a growing interest in the development of machine learning models in the ``predict-then-optimize'' (or ``decision-focused'', ``end-to-end learning'', etc.) setting, where models are trained in a way that is guided by the objectives of a downstream optimization task. See, for example, the works of \cite{bertsimas2020predictive,donti2017task,elmachtoub2021smart,kao2009directed,estes2019objective,ho2019data,notz2019prescriptive,kotary2021end}, and the references therein, among others. Prior work in this landscape has primarily been focused on the standard ``static'' setting where decision-making over time is not a critical aspect and there is no consumption of resources.
    In this work, we develop a new framework for integrating decision-focused learning methods, using predict-then-optimize losses, into the online decision-making task. Effectively utilizing the structure of the underlying optimization problem in the decision-making task leads to better decisions more quickly and a better management of the trade-off between immediate and future rewards, which we demonstrate in our numerical experiments. Specifically, we focus on an online contextual stochastic convex optimization problem where we would like to maximize the average (equivalently total) linear reward over time plus a concave utility function that measures the desirability of the average resource consumption levels so far. Our model also includes a convex feasibility constraint on the resource consumption vector. 
    At each time period, the decision-maker is given some contextual features that are associated with the coefficients of the unknown reward objective and resource consumption matrix. We present a ``meta-procedure'' for online-decision making which involves a prediction step as well as a decision step. In the prediction step, a model is trained for predicting the unknown coefficients based on the history of observed contexts and corresponding coefficients. In the decision step, we use these predictions to solve a linear optimization problem with a known feasible region to make a decision. 
    
    Due to the linear structure of the underlying optimization problem in our meta-procedure, we can apply the Smart Predict-then-Optimize (SPO) loss function and its SPO+ convex surrogate loss function developed by \cite{elmachtoub2021smart}. Importantly, the SPO loss function measures the regret of a prediction against the best decision in hindsight and is the ideal loss function to measure the error of the prediction models that we build. Unlike the standard SPO setting, we need to account for the trade-offs present due to the consumption of resources. To do so, we apply the customary technique of introducing dual variables and using primal-dual methods (see, e.g., \citet{agrawal2014fast}). As such, at each time period, we update a set of dual variables using the method of online mirror descent \citep{shalev2012online} and then we update the prediction model by minimizing a surrogate of the SPO loss on a dataset constructed by combining past observations with the current dual variables. A critical part of our contribution involves bridging convergence theory for primal-dual online methods with learning theory in the predict-then-optimize setting.
    In particular, we prove regret bounds for our overall algorithm that combine the $\mathcal{O}(T^{-1/2})$ convergence of online mirror descent with the convergence of the learning process, the rate of which depends on which surrogate loss function is used. To analyze the latter, we leverage risk bounds and related recent statistical results on the SPO loss and its surrogate loss functions \citep{elmachtoub2021smart, balghiti2019generalization,hu2020fast,ho2022risk,liu2021risk}. 
    These results enable us to use a general hypothesis class for fitting the prediction model. More specifically, we are no longer limited to the previously studied linear context or finite policy assumptions \citep{agrawal2016linear,badanidiyuru2018bandits}, and more complex machine learning models, such as random forests and neural networks, may be used. Our bounds hold in both hard and soft resource constraint cases, and we extend prior results using standard upper bound consumption constraints on each resource to arbitrary convex consumption constraints.
    On the experimental side, we examine the empirical performance of different loss functions in the prediction step of our algorithm. On multi-dimensional knapsack and longest path instances, we observe that the methods which perform best are those that account for both resource consumption via dual variables as well as the structure of the optimization problem via SPO-like loss functions.

    Online contextual learning problems have been previously studied under varying assumptions in several different settings. As in our setting, some of these works, including those that study online linear/convex programming, consider the case where full information is provided after a decision is made (see \cite{mahdavi2012trading,agrawal2014dynamic,agrawal2014fast,jenatton2016adaptive,liakopoulos2019cautious,balseiro2020best,li2021online,vera2021online,lobos2021joint}, among others). Other authors have considered bandit and related problems where only partial information is revealed after the decision is made (see \cite{agrawal2013thompson,agrawal2014bandits,badanidiyuru2014resourceful,agrawal2016linear,ferreira2018online,badanidiyuru2018bandits,pacchiano2021stochastic}, among others).

    \paragraph{Notation.}
    Let $\odot$ represent element-wise multiplication between two vectors. 
    Let $I_p$ denote the $p$ by $p$ identity matrix, and let $e$ denote the vector of all ones in the appropriate dimension.
    We will make use of a generic given norm $\vvta{\cdot}$ on $w \in \bbR^d$, as well as its dual norm  $\vvta{\cdot}_*$ which is defined by $\vvta{c}_* = \max_{w: \vvta{w} \le 1} c^T w$. With a slight abuse of notation, we also let $\vvta{\cdot}$ refer to a (possibly different) generic given norm on $v \in \bbR^m$, where which norm we are referring to is clear from the dimension of the corresponding vector.
    For any convex function $f(\cdot): F \to \bbR$ with its domain $F$, let $f^*(\cdot)$ denote its Fenchel conjugate function, namely $f^*(y) := \sup_{x \in F} \{y^T x - f(x)\}$. 
    We also make use of the big $\mathcal{O}$ notation to omit absolute constants.

\section{Online contextual convex optimization and preliminaries}\label{sec:online-convex}

    We now formally describe our online contextual stochastic convex optimization problem, which is prevalent in online decision-making. 
    We assume there are $T$ rounds of decision-making. At each round $t$, we make a decision $w_t \in \mathcal{S} \subseteq \bbR^d$, and associated with this decision is a ``budget consumption vector'' $v_t \in \bbR^m$.  
    Specifically, let $\mathcal{S} \subseteq \bbR^d$ denote the convex and compact feasible region of the decision variables and let $\mathcal{V} \subseteq \bbR^m$ denote the closed and convex feasible region of consumption vectors. 
    In addition, there is a ``utility function'' $u(\cdot): \bbR^m \to \bbR$, assumed to be $L$-Lipschitz and concave with $u(0) = 0$, that describes the consumption vector spending preferences of the decision-maker.
    We assume that we have full knowledge of $\mathcal{S}$, $\mathcal{V}$, and $u(\cdot)$.
    \begin{example}
        Consider a multi-dimensional knapsack problem where, in each round, the decision-maker receives $d$ different orders and may accept at most $k \leq d$ of them. Each accepted order receives a reward and consumes some of $m$ different resources. The amount of resources available per round is $b \in \bbR^m$. The selling price vector of any leftover resources is $y \in \bbR^m$. 
        In this case, the decision space is $\mathcal{S} = \{w \in \bbR^d : \sum_{j = 1}^d w_j \leq k, 0 \leq w \leq e\}$, the resource consumption feasible region is $\mathcal{V} = \{v \in \bbR^m: v \le b\}$, and the resource consumption utility function is $u(v) = y^T (b - v)^+$.
    \end{example}
    
    At time $t$, a tuple $(x_t, r_t, V_t)$ is identically and independently drawn from an unknown distribution $\bbP$, where $r_t \in \bbR^d$ denotes the reward vector, $V_t \in \bbR^{d \times m}$ denotes the budget consumption matrix, and $x_t \in \bbR^p$ denotes the feature vector which contains contextual information about $r_t$ and $V_t$. 
    The reward vector and consumption matrix are unknown when the decision $w_t \in \mathcal{S}$ needs to be made, while the context vector $x_t$ is given instead. 
    After the decision $w_t$ is made, the actual values of $r_t$ and $V_t$ will be revealed, and we will receive $r_t^T w_t$ as reward and also incur $v_t := V_t^T w_t$ consumption in the budget. 
    Let $\ravg$ and $\vavg$ denote the total averaged reward and consumption values, namely $\ravg := \frac1T \sum_{t=1}^T r_t^T w_t$ and $\vavg := \frac{1}{T} \sum_{t=1}^T v_t$. 
    The simultaneous objectives of the decision-maker are: (i) maximize the reward plus the utility of the budget consumption, i.e., $\max \{\ravg + u(\vavg)\}$, and (ii) ensure that the average consumption remains feasible, i.e., $\vavg \in \mathcal{V}$.
    
    \paragraph{Primal-dual formulation and meta-procedure.}
    Online decision-making problems, including online linear optimization and bandit problems with constraints, have been well-studied in the machine learning and operations research communities, wherein a common method to address budget consumption utility and/or feasibility constraints is with the primal-dual max-min form of the original problem. In our setting, we will need two sets of dual variables, $\theta$ and $\lambda$, to address both consumption utility and feasibility constraints simultaneously.
    Let $d_{\mathcal{V}}(\cdot)$ denote the distance function to the set $\mathcal{V}$, measured in the given generic norm $\|\cdot\|$, and let and $\zeta$ be the positive budget penalty parameter. That is, $d_{\mathcal{V}}(\cdot)$ is defined by $d_{\mathcal{V}}(v) := \inf_{\tilde{v} \in \mathcal{V}}\|\tilde{v} - v\|$.
    Then, for any values of $\ravg$ and $\vavg$ as defined above, we can consider a penalized version of the objective and its primal-dual reformulation using conjugate functions as follows:
    \begin{equation}\label{eq:lp-reform}
        \ravg + u(\vavg) - \zeta \cdot d_{\mathcal{V}}(\vavg) = \inf_{\lambda \in \Lambda, \theta \in \Theta} \{\ravg - (\lambda^T \vavg - (-u)^*(\lambda)) - \zeta \cdot (\theta^T \vavg - d_{\mathcal{V}}^*(\theta))\}, 
    \end{equation}
    where $\Lambda$ and $\Theta$ are the domains of the conjugate functions $(-u)^*(\cdot)$ and $d_\mathcal{V}^*(\cdot)$, respectively. 
    Note that $L$-Lipschitzness of $u(\cdot)$ and $1$-Lipschitzness of $d_{\mathcal{V}}(\cdot)$ imply that the domains satisfy $\Lambda \subseteq \{\lambda \in \bbR^m: \vvta{\lambda}_* \le L\}$ and $\Theta \subseteq \{\theta \in \bbR^m: \vvta{\theta}_* \le 1\}$. 
    The main benefit of the introduction of the dual variables is that the primal-dual objective becomes linear in the average reward and consumption whenever the dual variables are fixed. Thus, it is viable to apply an online descent method to the primal-dual min-max problem, which consists of two steps:  (i) making a decision by solving an optimization problem with a linear objective, and (ii) updating the dual variables via online descent. 
    \Cref{alg:meta} below presents a ``meta-procedure'' that combines these two steps with a \emph{prediction step} for predicting the reward vector and consumption matrix based on the context $x_t$. 
    We will specify the precise methods for the prediction model and dual variables update later in \Cref{sec:practical}. 
    
    \IncMargin{1em}\begin{algorithm}[H]
		    Observe feature vector $x_t$\; 
		    Make predictions $(\hat{r}_t$, $\hat{V}_t) \gets g_t(x_t)$ for reward and consumption\;
		    Make the decision $w_t \gets \arg \max_{w \in \mathcal{S}} \{(\hat{r}_t - \hat{V}_t \lambda_t - \zeta \cdot \hat{V}_t \theta_t)^T w\}$\; 
		    Observe realized reward $r_t$ and consumption $V_t$\;
            Update dual variables $\theta_{t+1}, \lambda_{t+1}$, and prediction model $g_{t+1}(\cdot)$\;
		\caption{A ``meta-procedure'' for online contextual decision-making at time $t$}
		\label{alg:meta}
	\end{algorithm}\DecMargin{1em}
	
	\paragraph{Benchmark.}
	In the theoretical part of this work, we compare the performance of the proposed online algorithm against the performance of the optimal static policy, \ie a policy which knows the distribution $\bbP$ but only requires to satisfy the resource budget constraints in expectation. 
	The formal definition of the optimal static policy is given below.
	\begin{definition}\label{def:static-policy}
	    Consider any static policy $\pi(\cdot): \mathcal{X} \to \mathcal{S}$, and define the expected reward and resource consumption of $\pi(\cdot)$ as $\mathrm{rew}(\pi) := \bbE_{(x, r, V) \sim \bbP} [r^T \pi(x)]$ and $\mathrm{con}(\pi) := \bbE_{(x, r, V) \sim \bbP} [V^T \pi(x)]$. 
	    Also, define the optimal static reward as the supremum of all feasible static policies, namely 
	    \begin{equation*}
    	    \mathrm{OPT} := \sup_\pi \, \{\mathrm{rew}(\pi) + u(\mathrm{con}(\pi))\}, \quad \tn{s.t.} \, \mathrm{con}(\pi) \in \mathcal{V}. 
	    \end{equation*}
	\end{definition}
	
	Another benchmark would be the optimal adaptive policy, which knows the distribution $\bbP$ and also takes the history into account. It turns out that the expected total reward of this adaptive policy is upper bounded by the one from the static one \citep{agrawal2016linear}. As has been considered in similar settings (see, for example, \cite{devanur2011near,badanidiyuru2013bandits,agrawal2016linear}), we will therefore work with the optimal static policy benchmark defined above.

\section{An online algorithm using predict-then-optimize and mirror descent}\label{sec:practical}
    
    In this section, we specify the details for the prediction model and dual variables updates in \Cref{alg:meta}. In particular, we first describe the predict-then-optimize methodology for learning the prediction model and then describe the online mirror descent method for updating the dual variables. 
    
    \subsection{Prediction model updating and the SPO loss}\label{subsec:spo}
    In order to obtain a model for predicting reward vectors and consumption matrices, namely a prediction function $g : \bbR^p \to \bbR^d \times \bbR^{d \times m}$, we may leverage machine learning methods to learn the underlying distribution $\bbP$ from previously observed data $\{(x_1, r_1, V_1), \dots, (x_{t-1}, r_{t-1}, V_{t-1})\}$, which are assumed to be independent samples from $\bbP$. Notice that the optimization subroutine to determine $w_t$ in \Cref{alg:meta} involves a \emph{linear} objective function. Ideally, with full knowledge of the distribution $\bbP$, one would determine $w_t$ by solving the optimization problem
    \begin{equation}\label{eq:lp-spo}
        \max_{w \in \mathcal{S}} \, \bbE_{r, V \sim \bbP(\cdot \vert x)} \lrr{(r - V \lambda - \zeta \cdot V \theta)^T w} = \max_{w \in \mathcal{S}} \, \bbE_{r, V \sim \bbP(\cdot \vert x)} [r - V \lambda - \zeta \cdot V \theta]^T w. 
    \end{equation}
    Due to the linearity of the objective, the above equation implies that it is sufficient to learn the conditional expectation of the ``linear cost vector'' $c = r - V \lambda - \zeta \cdot V \theta$. Thus $g(x)$ can be thought of as providing estimates of $\bbE[r \vert x]$ and $\bbE[V \vert x]$, which are then plugged into the corresponding linear optimization problem with feasible region $\mathcal{S}$.
    This setting is essentially a parametric variant of the the predict-then-optimize framework, where the dual variables $\omega := (\lambda, \theta)$ are parameters that specify a linear cost vector that we would like to learn.
    In the usual ``static'' setting without parameters, \cite{elmachtoub2021smart} introduced and studied the SPO loss function, which characterizes the excess cost, or decision error, incurred when making a suboptimal decision due to an imprecise objective cost vector prediction.
    Let us now adapt the SPO loss to our setting. In the usual case, given a predicted cost vector $\hat{c}$ and a realized cost vector $c$, the SPO loss for a linear optimization problem in maximization format is defined as $\lspo(\hat{c}, c) := c^T(w^\ast(c) - w^\ast(\hat{c}))$ where $w^\ast(\cdot)$ is an optimization oracle for $\mathcal{S}$ satisfying $w^\ast(c) \in \arg \max_{w \in \mathcal{S}}\left\{c^Tw\right\}$. In our setting, we need to consider a parameteric variant of the SPO loss where the dual variables are parameters affecting the objective cost vectors.
    In particular, given a prediction $\hat{\mu} := (\hat{r}$, $\hat{V})$, realization $\mu := (r, V)$, dual variables $\omega := (\lambda, \theta)$, as well as fixed budget penalty parameter $\zeta > 0$, the SPO loss of the optimization problem \eqref{eq:lp-spo} is defined as 
    \begin{equation*}
        \lspo(\hat{\mu}, \mu; \omega) := (r - V \lambda - \zeta \cdot V \theta)^T (w^*(\mu; \omega) - w^*(\hat{\mu}; \omega)), 
    \end{equation*}
    where $w^*(\cdot)$ denotes the optimization oracle, which is defined as a function satisfying 
    \begin{equation*}
        w^*(\mu; \omega) \in \arg \max_{w \in \mathcal{S}} \lrrr{(r - V \lambda - \zeta \cdot V \theta)^T w}, \text{ for all } \mu \in \bbR^d \times \bbR^{d \times m} \text{ and } \omega \in \Lambda \times \Theta.
    \end{equation*}
    Since the SPO loss function is usually non-convex and even possibly non-continuous, several surrogate loss functions have been introduced. For example, \cite{elmachtoub2021smart} introduce the SPO+ loss function that accounts for the structure of $\mathcal{S}$ when training the prediction model. This loss function is defined as $\lspop(\hat{c}, c) := \max_{w \in \mathcal{S}} \{(c - 2\hat{c})^T w\} + 2 \hat{c}^T w^*(c) - c^T w^*(c)$. On the other hand, more standard prediction error loss functions, like the squared $\ell_2$ loss of the linear cost vector, may be considered.
    Let $\ell(\cdot, \cdot) : \bbR^d \times \bbR^d \to \bbR$ be any surrogate loss function of the standard SPO loss, which takes cost vector inputs, including possibly itself. Then, just as we defined an extension of the SPO loss to our setting with dual variables, we can also extend the surrogate loss $\ell$ by defining $\ell(\hat{\mu}, \mu; \omega) := \ell(\hat r - \hat V \lambda - \zeta \cdot \hat V \theta, r - V \lambda - \zeta \cdot V \theta)$. 
    Given a surrogate loss function, we use empirical risk minimization to update the prediction model $g_t$ at each step of our online decision-making method.
    Specifically, let $\mathcal{H}$ refer to a hypothesis class of predictor functions mapping features $x$ to predictors $(\hat{r}, \hat V)$ of the reward vector and resource consumption matrix. Then, the prediction model used at iteration $t$ is chosen by $g_t \gets \arg \min_{g \in \mathcal{H}} \sum_{s=1}^{t-1} \ell(g(x_s), \mu_s; \omega_t)$. 
    It is worthwhile to notice that the dual variables used in the loss function are those of the current iteration instead of the previous ones. Intuitively, the current set of dual variables $\omega_t$ is closer to the optimal dual variables of the offline expected problem and hence it leads to a better prediction model. 
    
    A desirable property of the surrogate loss $\ell$ is that the empirical risk minimizer for $\ell$ has small excess risk with respect to the SPO loss. This property is formalized below in \Cref{assumption:spo-risk-bounds}, wherein we measure the excess risk by comparing with the ground truth conditional expectation function of the reward vector and resource consumption matrix, namely $g^*(x) := \bbE_{\mu \sim \bbP(\cdot \vert x)} [\mu]$.
    Let us also define the expected risk functions for the two losses by $\rspo(g; \omega) := \bbE_{(x, \mu) \sim \bbP} [\lspo(g(x), \mu; \omega)]$ and $R_\ell(g; \omega) := \bbE_{(x, \mu) \sim \bbP} [\ell(g(x), \mu; \omega)]$.
    
    
    \begin{assumption}\label{assumption:spo-risk-bounds}
        There exist constants $\kappa_{\mathrm{risk}}, \alpha > 0$ such that, for any integer $n > 0$ and uniformly over all dual variables $\omega \in \Lambda \times \Theta$, the empirical surrogate loss optimal predictor $\hat{g}^n := \arg \min_{g \in \mathcal{H}} \lrrr{\sum_{i=1}^n \ell(g(x_i), \mu_i; \omega)}$
        satisfies the following excess true SPO risk guarantee:
        \begin{equation*}
            \bbE\left[\sup_{\omega \in \Lambda \times \Theta} \left\{\rspo(\hat{g}^n; \omega) - \rspo(g^\ast; \omega)\right\}\right] \leq \kappa_{\mathrm{risk}} \cdot n^{-\alpha}, 
        \end{equation*}
        where the expectation is taken with respect to i.i.d.\ samples $\{(x_i, \mu_i)\}_{i=1}^n$ drawn from $\bbP$.
    \end{assumption}
    
    In general, the rate of convergence $\alpha$ and the constant $\kappa_{\mathrm{risk}}$ in \Cref{assumption:spo-risk-bounds} depend on the properties of the surrogate loss function, the decision feasible region $\mathcal{S}$, the underlying distribution $\bbP$, and the complexity of the hypothesis class $\mathcal{H}$. 
    \Cref{assumption:spo-risk-bounds} is closely tied to uniform calibration properties of the surrogate loss $\ell$ with respect to the SPO loss. In fact, the following remark demonstrates that uniform calibration and an excess risk bound for $\ell$ are sufficient conditions for \Cref{assumption:spo-risk-bounds}.
    \begin{remark}\label{remark:calibration}
        Suppose that the the surrogate loss function $\ell(\cdot, \cdot)$ is uniformly calibrated with respect to the true SPO loss in the standard setting. Namely, for some constants $\kappa_1$ and $\beta$, for any distribution $\tilde{\bbP}$ over cost vectors $c$, we have
        \begin{equation*}
            \bbE_{c \sim \tilde{\bbP}}[\ell(\hat{c}, c) - \ell(\bbE[c], c)] \le \kappa_1 \cdot \epsilon^\beta \, \Rightarrow \, \bbE_{c \sim \tilde{\bbP}}[\lspo(\hat{c}, c) - \lspo(\bbE[c], c)] \le \epsilon,
        \end{equation*}
        for all $\hat{c} \in \bbR^d$ and $\epsilon > 0$.
        Suppose further that the empirical surrogate loss optimal predictor has an excess risk bound that holds uniformly over all dual variables $\omega \in \Lambda \times \Theta$, i.e., for some constants $\kappa_2$ and $\gamma$ we have 
        \begin{equation*}
            \bbE\left[\sup_{\omega \in \Lambda \times \Theta}\left\{R_\ell(\hat{g}^n; \omega) - R_\ell(g^\ast; \omega)\right\}\right] \le \kappa_2 \cdot n^{-\gamma}, 
        \end{equation*}
        for any $n > 0$ and where the expectation is taken with respect to i.i.d.\ samples $\{(x_i, \mu_i)\}_{i=1}^n$ drawn from $\bbP$. Then, under both of these conditions, it holds that 
        \begin{equation*}
            \bbE\left[\sup_{\omega \in \Lambda \times \Theta}\left\{\rspo(\hat{g}^n; \omega) - \rspo(g^\ast; \omega)\right\}\right] \le (\kappa_2 / \kappa_1)^{1 / \beta} \cdot n^{-\gamma / \beta}. 
        \end{equation*}
    \end{remark}
    
    We note that the two conditions in \Cref{remark:calibration} often hold for many choices of surrogate loss $\ell$ and hypothesis classes $\mathcal{H}$. Indeed, recent works including \cite{ho2022risk,hu2020fast,liu2021risk} 
    have examined sufficient conditions under which uniform calibration holds for the SPO loss. Some conditions require an additional restriction on the class of distributions $\tilde{\bbP}$ over cost vectors, but such restrictions will often be satisfied in practice in our setting. Furthermore, for most common choices of surrogate losses $\ell$, e.g., convex and Lipschitz losses like squared $\ell_2$ and SPO+, the required excess risk bound will hold. Indeed, for most common surrogates, due to the boundedness of the dual variable domains $\Lambda$ and $\Theta$, boundedness of the Rademacher complexity of $\mathcal{H}$ would be a sufficient condition to ensure that the bound holds uniformly over the dual variables.
    
    \subsection{Dual variable updates with online mirror descent}\label{subsec:dual-update}
    In this section, we describe how we use online mirror descent method to update the dual variables $\lambda_{t+1}$ and $\theta_{t+1}$. 
    Based on the reformulation \ref{eq:lp-reform}, we define convex functions $\xi_t(\lambda) := -v_t^T \lambda + (-u)^*(\lambda)$ and $\phi_t(\theta) := -v_t^T \theta + d_{\mathcal{V}}^*(\theta)$, where we recall that $v_t = V_t^T w_t$. Note that the domains of $\xi_t(\cdot)$ and $\phi_t(\cdot)$ are the same as the domains of the corresponding conjugate functions, i.e., the previously defined sets $\Lambda$ and $\Theta$, respectively.
    Let $h_{\Lambda}(\cdot), h_{\Theta}(\cdot)$ be differentiable and $1$-strongly convex (with respect to the dual norm of the norm on $v \in \bbR^m$) functions on $\Lambda, \Theta$, respectively, and let $B_{h_{\Lambda}}(\cdot, \cdot), B_{h_{\Theta}}(\cdot, \cdot)$ denote their respective Bregman distances. For example, $B_{h_{\Lambda}}(\cdot)$ is defined by $B_{h_{\Lambda}}(\lambda_1, \lambda_2) := h_{\Lambda}(\lambda_1) - h_{\Lambda}(\lambda_2) - \nabla h_{\Lambda}(\lambda_2)^T(\lambda_1 - \lambda_2)$. 
    Then, online mirror descent uses the following update schemes for the dual variables:
    \begin{equation*}
        \lambda_{t+1} \gets \arg \min_{\lambda \in \Lambda} \{\eta_{\lambda} \nabla \xi_t(\lambda_t)^T \lambda + B_{h_{\Lambda}}(\lambda, \lambda_t)\}, \quad \theta_{t+1} \gets \arg \min_{\theta \in \Theta} \{\eta_{\theta} \nabla \phi_t(\theta_t)^T \theta + B_{h_{\Theta}}(\theta, \theta_t)\}, 
    \end{equation*}
    where $\eta_\lambda, \eta_\theta > 0$ are the ``step-size'' parameters. Here we abuse notation slightly and let $\nabla$ refer to any subgradient of the functions $\xi_t(\cdot)$ and $\phi_t(\cdot)$. We assume that we can efficiently calculate such subgradients, i.e., by evaluating the subproblems defining the conjugate functions. We also need to be able to efficiently calculate the solution of the above subproblems, which depends on the structure of the Bregman distances. For example, recall that the online mirror descent method is the same as the online projected gradient descent method when the norm and Bregman distances are Euclidean.
    The following lemma provides an upper bound of the regret from the online mirror descent method. 
    \begin{lemma}\label{lemma:regret-mirror}[Theorem 2.15 in \cite{shalev2012online}]
        Let $D_\Lambda, D_\Theta$ be upper bounds on the Bregman distances so that $B_{h_{\Lambda}}(\lambda_1, \lambda_2) \leq D_\Lambda$ and $B_{h_{\Theta}}(\theta_1, \theta_2) \leq D_\Theta$ for all $\lambda_1, \lambda_2 \in \Lambda$ and $\theta_1, \theta_2 \in \Theta$. Let $G_\Lambda, G_\Theta$ be upper bounds on norms of the subgradients so that $\|\nabla \xi_t(\lambda_t)\| \leq G_\Lambda$ and $\|\nabla \phi_t(\theta_t)\| \leq G_\Theta$ for all $t = 1, \dots, T$. 
        If we use the constant step-sizes $\eta_\lambda \gets \frac{D_\Lambda}{G_\Lambda \sqrt{T}}$ and $\eta_\theta \gets \frac{D_\Theta}{G_\Theta \sqrt{T}}$, then for all $\lambda \in \Lambda$ and all $\theta \in \Theta$ it holds that 
        \begin{equation*}
            \mathcal{R}_\lambda(T) := \sum_{t=1}^T (\xi_t(\lambda_t) - \xi_t(\lambda)) \le 2 G_\Lambda \sqrt{D_\Lambda T}, \quad \mathcal{R}_\theta(T) := \sum_{t=1}^T (\phi_t(\theta_t) - \phi_t(\theta)) \le 2 G_\Theta \sqrt{D_\Theta T}. 
        \end{equation*}
    \end{lemma}
    
    Note that, due to the bounds on the norm of dual variables in the domains $\Lambda$ and $\Theta$, the Bregman constants usually satisfy that $\sqrt{D_\Lambda} = \mathcal{O}(L)$ (recall that $L$ is the Lipschitz constant for $u(\cdot)$), and $\sqrt{D_\Theta} = \mathcal{O}(1)$. Indeed, in the Euclidean case this is guaranteed to be true. 
    Now we have specified the update rules for both the prediction model and the dual variables, we can fully state an implmentable version of the meta algorithm presented previously. This implementable algorithm to solve our online decision-making problem is presented in \Cref{alg:practical}. 
    
    \IncMargin{1em}\begin{algorithm}[H]
		\SetKwData{Left}{left}\SetKwData{This}{this}\SetKwData{Up}{up}
		\SetKwFunction{Union}{Union}\SetKwFunction{FindCompress}{FindCompress}
		\SetKwInOut{Input}{input}\SetKwInOut{Output}{output}
		\Input{Budget penalty parameter $\zeta$ and surrogate loss function $\ell(\cdot, \cdot; \cdot)$.}
			
		Initialize dual variables $\theta_1, \lambda_1$, and prediction model $g_1(\cdot)$\; 
		\For {$t = 1, \dots, T$}{
		    Observe feature vector $x_t$\; 
		    Make predictions $(\hat{r}_t$, $\hat{V}_t) \gets g_t(x_t)$ for reward and consumption\;
		    Make the decision $w_t \gets \arg \max_{w \in \mathcal{S}} \{(\hat{r}_t - \hat{V}_t \lambda_t - \zeta \cdot \hat{V}_t \theta_t)^T w\}$\; 
		    Observe realized reward $r_t$ and consumption $V_t$\;
            Update dual variable $\lambda_{t+1} \gets \arg \min_{\lambda \in \Lambda} \lrrr{\eta_{\lambda} \nabla \xi_t(\lambda_t)^T \lambda + B_{h_{\Lambda}}(\lambda, \lambda_t)}$\; 
            Update dual variable $\theta_{t+1} \gets \arg \min_{\theta \in \Theta} \lrrr{\eta_{\theta} \nabla \phi_t(\theta_t)^T \theta + B_{h_{\Theta}}(\theta, \theta_t)}$\; 
            Update prediction model $g_{t+1} \gets \arg \min_{g \in \mathcal{H}} \{\sum_{s=1}^{t} \ell(g(x_s), \mu_s; \omega_{t+1})\}$\;
		}
		\caption{An implementable algorithm for online contextual decision-making}
		\label{alg:practical}
	\end{algorithm}\DecMargin{1em} 
    
\section{Regret bounds and analysis}\label{sec:regret}

    In this section, we present the regret analysis of \Cref{alg:practical} in two cases:  hard and soft constraints.

    \paragraph{Hard constraints.}
    We assume the starting point of the budget consumption, which we naturally assume to be the zero vector without loss of generality, is inside the consumption feasible region $\mathcal{V}$. The hard constraints case is when we add a stopping condition to \Cref{alg:practical} that terminates whenever the current resource consumption vector violates the constraints enforced by $\mathcal{V}$, i.e., whenever it leaves the feasible region $\mathcal{V}$. 
    We introduce a stopping time $\tau$ that is the first time before time $T$ that the constraints are violated, i.e., $\tau := \min \{t \le T: \frac1T \sum_{s=1}^t V_s^T w_s \not\in \mathcal{V}\}$. 
    Most previous works with hard resource constraints, for example, \cite{agrawal2016linear} and \cite{li2021online}, consider the case of an upper bound budget constraint for each resource, i.e., $\mathcal{V} = \{v: v \le b\}$ for some $b \in \bbR^m$. In such cases, the online algorithm will terminate immediately whenever it violates any of the resource constraints.
    In contrast, let $B_{\mathcal{V}}$ denote the distance from zero to the boundary of the generic closed and convex set $\mathcal{V}$. The constant $B_{\mathcal{V}}$ will be important in our analysis as it demonstrates how the structure of $\calV$ affects the constant in the regret bound. In fact, $B_{\mathcal{V}}$ precisely generalizes constants appearing in previous works considering only upper bound constraints, for example, \cite{agrawal2016linear}. Indeed, when $\mathcal{V} = \{v: v \le b\}$ for some $b \in \bbR^m$ with each component positive, then it holds that $B_{\mathcal{V}} = \min_{i = 1, \ldots, m}b_i > 0$.
    We make the following boundedness assumption on the distribution.
    \begin{assumption}\label{assumption:bounded-consumption}
        Suppose there exists a constant $D_v \ge 1$, such that for any $w \in \mathcal{S}$, it holds that $\vvta{V^T w} \le D_v$ with probability $1$. Let the constant $\kappa_{\mathrm{MD}} := D_v \cdot (\zeta \sqrt{D_{\Theta}} + \sqrt{D_{\Lambda}})$.
    \end{assumption}
    In the hard constraints case, let $\ravg^\tau$ and $\vavg^\tau$ denote the total averaged reward and consumption with stopping time $\tau$, namely $\ravg^\tau := \frac{1}{\tau} \sum_{t=1}^\tau r_t^T w_t$ and $\vavg^\tau := \frac{1}{\tau} \sum_{t=1}^\tau v_t$. 
    Below we provide our main theorem in the hard constraint case, which provides the regret bound of \Cref{alg:practical}.
    
    \begin{theorem}\label{thm:regret-hard}
        Suppose that Assumptions \ref{assumption:spo-risk-bounds} and \ref{assumption:bounded-consumption} hold, and that the budget penalty parameter $\zeta$ satisfies $\zeta \ge \frac{\mathrm{OPT}}{B_{\mathcal{V}}}$. Then \Cref{alg:practical} has the following guarantee: 
        \begin{equation*}
            \mathrm{OPT} - \bbE [\ravg^\tau + u(\vavg^\tau)] \le \kappa_{\mathrm{MD}} \cdot \mathcal{O} (T^{-1/2})  + \kappa_{\mathrm{risk}} \cdot \mathcal{O} (T^{- \alpha}). 
        \end{equation*}
    \end{theorem}
    
    We remark that the constant $\kappa_{\mathrm{MD}}$ will usually satisfy $\kappa_{\mathrm{MD}} = D_v \cdot (\zeta\mathcal{O}(1) + \mathcal{O}(L))$ for most choices of Bregman functions. 
    In general, given the required lower bound of $\zeta$ in \Cref{thm:regret-hard}, the best value of the constant $\kappa_{\mathrm{MD}}$ will be $\mathcal{O}(D_v \cdot (\frac{\mathrm{OPT}}{B_{\mathcal{V}}} \sqrt{D_{\Theta}} + \sqrt{D_{\Lambda}}))$ whenever we are able to set $\zeta = \mathcal{O}(\frac{\mathrm{OPT}}{B_{\mathcal{V}}})$.
    The dependence on the term $\frac{\mathrm{OPT}}{B_{\mathcal{V}}}$ in the regret bound is natural, since, if the budget starting point is very close to the boundary of the feasible set (or equivalently, in the budget upper bound case, one of the resource budget values is very close to zero), then \Cref{alg:practical} is likely to terminate in the first several iterations leading to a poor regret bound. 
    
    \paragraph{Soft constraints.}
    In this case, we treat the budget consumption feasibility constraint set $\mathcal{V}$ as a soft constraint, that is, we want to minimize the infeasibility of the consumption instead of terminating the online algorithm whenever we violate the constraint. 
    This case allows for the possibility that the starting point of budget consumption is infeasible, for example, when the feasible region $\mathcal{V}$ consists of both lower and upper bound constraints. 
    The following assumption is required for the regret analysis in this case.
    \begin{assumption}\label{assumption:opt-relax}
        Let $\mathrm{OPT}^{\epsilon}$ denote the optimal objective value of the following relaxed problem: 
        \begin{equation*}
            \mathrm{OPT}^\epsilon := \sup_\pi \, \{\mathrm{rew}(\pi) + u(\mathrm{con}(\pi))\}, \quad \tn{s.t.} \, d_{\mathcal{V}}(\mathrm{con}(\pi)) \le \epsilon. 
        \end{equation*}
        We assume there exists a constant $\zeta_{\mathrm{OPT}}$ such that $\mathrm{OPT}^\epsilon \le \mathrm{OPT} + \zeta_{\mathrm{OPT}} \cdot \epsilon$ for all $\epsilon > 0$. 
    \end{assumption}
    The constant $\zeta_{\mathrm{OPT}}$ in \Cref{assumption:opt-relax} can be interpreted as a subgradient of the concave function $\mathrm{OPT}^{\epsilon}$, which can be demonstrated to exist under standard regularity conditions. For example, in the supplementary materials, we demonstrate that \Cref{assumption:opt-relax} holds when $\mathcal{V}$ contains a feasible policy in its interior. We are now ready to present the main theorem in the soft constraint case, which demonstrates the convergence rate in terms of both the objective and the distance to feasibility in resource consumption.
    \begin{theorem}\label{thm:regret-soft}
        Suppose that Assumptions \ref{assumption:spo-risk-bounds}, \ref{assumption:bounded-consumption}, and \ref{assumption:opt-relax} hold. Then, \Cref{alg:practical} has the following guarantee: 
        \begin{equation*}
            \mathrm{OPT} - \bbE [\ravg + u(\vavg)] \le \kappa_{\mathrm{MD}} \cdot \mathcal{O} (T^{-1/2}) + \kappa_{\mathrm{risk}} \cdot \mathcal{O} (T^{- \alpha}). 
        \end{equation*}
        If additionally the budget penalty parameter $\zeta$ satisfies $\zeta \ge 2(\zeta_{\mathrm{OPT}} + \frac{\sqrt{D_\Lambda}}{\sqrt{D_\Theta}} + \frac{\kappa_{\mathrm{risk}}}{D_v \sqrt{D_\Theta}})$, it holds that 
        \begin{equation*}
            \bbE [d_\mathcal{V}(\vavg)] \le D_v \sqrt{D_\Theta} \cdot \mathcal{O} (T^{-1/2}) + \mathcal{O} (T^{-\alpha}). 
        \end{equation*}
    \end{theorem}
    
    The updating of the prediction model at each iteration of \Cref{alg:practical} may present a computational burden in some situations. 
    To address this issue, we develop a variant of our algorithm that only updates the prediction model periodically at a sublinear rate. We demonstrate in \Cref{appendix:beta-efficient} that similar regret bounds hold for this variant.
    
    To give some intuition of the proofs of \Cref{thm:regret-hard} and \Cref{thm:regret-soft}, we remark that the total regret of the online algorithm can be divided into two parts:  (i) the regret from the learning of the prediction model, and (ii) the regret from the suboptimality of the dual variables used in each iteration. 
    In the supplementary materials, we present two lemmas to bound each type of regret. \Cref{lemma:regret-prediction} bounds the regret due to learning, in particular the expected accumulative errors of the online decision $w_t$ due to imperfect predictions, which can be bounded in a sublinear fashion based on \Cref{assumption:spo-risk-bounds}.
    To bound the regret due to suboptimality of the dual variables, we use the regret bound of online mirror descent method in \Cref{lemma:regret-mirror} and properties of the Fenchel conjugate functions, which, with a few additional steps, yield \Cref{lemma:regret-hard} and \Cref{lemma:regret-soft}. In the hard and soft cases respectively, these two Lemmas provide guarantees of the decisions $w_t$ from \Cref{alg:practical} against any feasible static policy.

\section{Computational experiments}\label{sec:experiment}

    We present computational results of synthetic dataset experiments wherein we empirically examine the performance of \Cref{alg:practical} using different surrogate loss functions for training prediction models. 
    We focus on two classes of prediction models to represent different levels of model complexity: {\em (i)} linear models, and {\em (ii)} two-layer neural networks with 128 neurons in the hidden layer. 
    We compare the performance of the empirical minimizer of the following three different loss functions: \emph{(i)} the previously defined SPO+ loss function, \emph{(ii)} the least squares (squared $\ell_2$) loss function of the linear objective $\vvta{(\hat{r} - \hat{V} \lambda - \zeta \cdot \hat{V} \theta) - (r - V \lambda - \zeta \cdot V \theta)}_2^2$, and \emph{(iii)} the least squares loss function of predictions $\vvta{\hat{r} - r}_2^2 + \vvta{\hat{V} - V}_F^2$. 
    Note that the three loss functions utilize different levels of information: the loss function \emph{(iii)} does not use the dual variables and can be viewed as purely learning the relationship between reward, consumption, and feature vectors. The loss function \emph{(ii)} does not utilize the structure of the decision feasible region $\mathcal{S}$ and can viewed as purely learning the relationship between the linear objectives and feature vectors. 
    We also compare with the following three methods as benchmarks:
    \emph{(i)} the sample average approximation (SAA) method, where we use the empirical averages of past observations of $r_t, V_t$ as the prediction $\hat{r}_t, \hat{V}_t$ in \Cref{alg:practical}, 
    \emph{(ii)} the true model, where we use the true (but unknown in practice) conditional expectations $\bbE[r_t \vert x_t], \bbE[V_t \vert x_t]$ as the prediction, and
    \emph{(iii)} the hindsight model, where we use the realization $r_t, V_t$ as the prediction.
    Note that \emph{(ii)} and \emph{(iii)} are not implementable in practice, because \emph{(ii)} uses the unknown true conditional expectations and \emph{(iii)} uses the realized values $r_t, V_t$ that are not available at decision-making time. We expect \emph{(iii)} to perform best and thus we define the ``relative regret'' of an online algorithm as $1 - \mathrm{OBJ} / \mathrm{OBJ}^*$ where $\mathrm{OBJ} := \ravg + u(\vavg)$ is the observed value of the objective function of the online algorithm and $\mathrm{OBJ}^*$ is the corresponding value for the hindsight policy.
    
    In this section, we consider multi-dimensional knapsack problem instances, where the goal is to maximize total reward collected. There is no utility function, the resource consumption feasible region is $\mathcal{V} = \{v: v \le b \cdot e\}$ for constant $b > 0$ and the online algorithm must terminates immediately when any of the resource constraints are violated. 
    In our simulations, the relationship between the true reward vector $r$, true resource consumption matrix $V$, and its context vector $x$ is given by $\tn{vec}(r, V) \gets \xi^{\mathrm{deg}}(W x) \odot \epsilon$, where $\tn{vec}(\cdot)$ is the matrix vectorization function, $\xi^{\mathrm{deg}}$ is a polynomial kernel mapping of degree $\mathrm{deg}$, $W \in \bbR^{d(m+1) \times p}$ is a fixed weight matrix, and $\epsilon \in \bbR^{d(m+1)}$ is a multiplicative noise term. 
    The data generation details are provided in \Cref{appendix:experiment}, we set the polynomial degree to 6 in this experiment, and we run $40$ independent trials for each value of the time horizon length.
    \Cref{fig:multi_knapsack_sample} displays the empirical performance of each method. We observe that when the hypothesis class is linear predictors, \ie the ground truth model is not in the hypothesis class, the pure prediction error method has similar performance as naive SAA, while least squares applied to the linear cost performs slightly better. When the hypothesis class is a neural net, the pure prediction method does better. In all cases, the SPO+ loss performs best and closest to the true model. These results demonstrate that a loss function that properly accounts for the dual variables can improve performance, but performance is improved twofold by a loss function that accounts for the dual variables \emph{and} the underlying structure of the decision feasible region $\mathcal{S}$.
    Further details on our experimental setup and additional plots are provided in \Cref{appendix:experiment}, as well as additional results on longest path instances that use a non-trivial utility function $u(\cdot)$.
    
    \begin{figure}
        \centering
        \includegraphics[width=\textwidth]{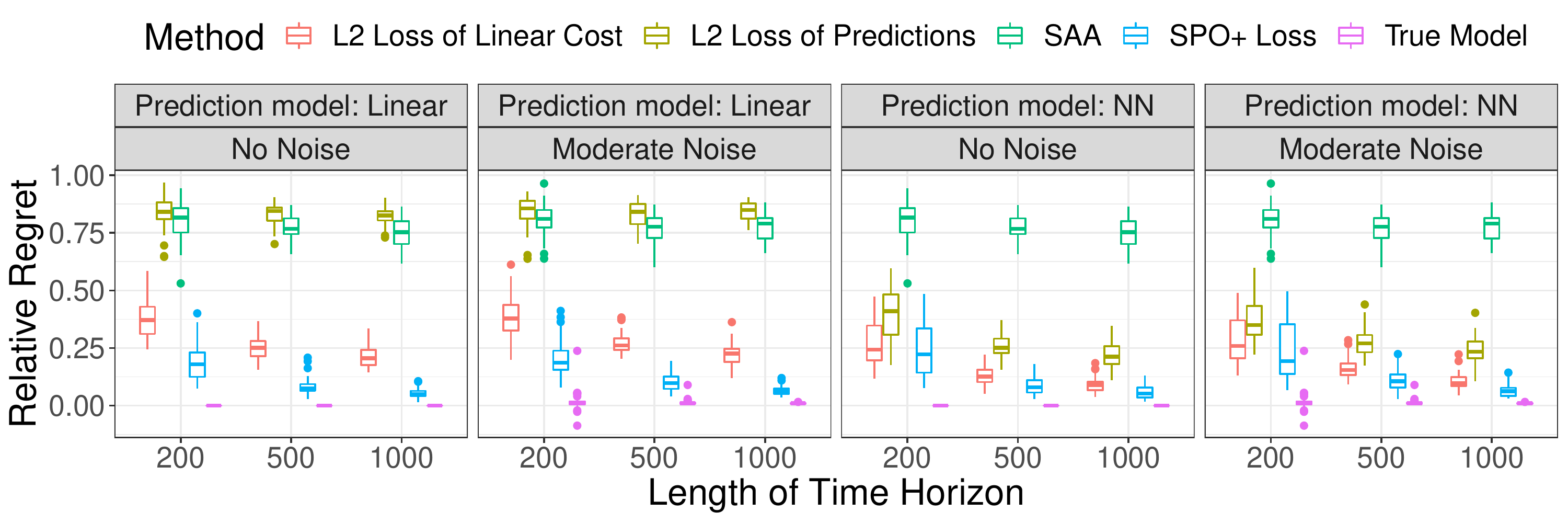}
        \caption{Relative regret for different loss functions on multi-dimensional knapsack instances.}
        \label{fig:multi_knapsack_sample}
    \end{figure}

\section{Conclusions and further directions}\label{sec:conclusion}

    In this work, we develop an algorithm for an online contextual convex optimization problem. 
    We analyze the regret of our algorithm by combining learning theory in the predict-then-optimize setting with the theory of online mirror descent, and as such we able to accommodate a very broad class of prediction models. We demonstrate empirically that the SPO+ performs best in our algorithm, across multiple instances on both linear and single-layer neural network hypothesis classes.
    In some practical problems, the decision-maker only receives partial information about the true reward and resource consumption, depending on the decision that is made. Extending our algorithm and its analysis to this partial information/bandit setting is an interesting and promising direction.

\bibliographystyle{abbrvnat}
\bibliography{ref}

\appendix

\section{Proofs in \Cref{sec:regret}}\label{appendix:proof-regret}

    In this section, we provide some useful lemmas and the complete proofs of \Cref{thm:regret-hard} and \Cref{thm:regret-soft}. 
    For convenience in our proofs, let us introduce some new notation. Let $\calF_{t-1}$ denote the $\sigma$-field of information revealed up to the start of iteration $t$, i.e., the $\sigma$ field of $(x_1, \mu_1), \dots, (x_{t-1}, \mu_{t-1})$.
    Let $\calRup_\lambda(T) := 2 G_\Lambda \sqrt{D_\Lambda T}$ and $\calRup_\theta(T) := G_\Theta \sqrt{D_\Theta T}$ denote the upper bounds of the regret of online mirror descent from \Cref{lemma:regret-mirror}.
    
    \subsection{Useful lemmas}
    We start off by providing an upper bound on the regret from learning the ground truth model when \Cref{assumption:spo-risk-bounds} holds. 
    
    \begin{lemma}\label{lemma:regret-prediction}
        Suppose \Cref{assumption:spo-risk-bounds} holds. For any policy $\pi(\cdot): \mathcal{X} \to \mathcal{S}$ and $T \geq 1$, \Cref{alg:practical} satisfies
        \begin{equation*}
            \mathcal{R}_g(T) := \bbE \lrr{\sum_{t=1}^T (r_t - V_t \lambda_t - \zeta \cdot V_t \theta_t)^T (\pi(x_t) - w_t)} \le \kappa_{\mathrm{risk}} \cdot \mathcal{O}(T^{1-\alpha}). 
        \end{equation*}
    \end{lemma}
    
    \begin{proof}
        Let $\bar{w}_t := w^*(g^*(x_t); \omega_t)$, where recall that $g^*(x) := \bbE_{\mu \sim \bbP(\cdot \vert x)} [\mu]$ is the Bayes estimator (i.e., the ground truth model). 
        Since \Cref{assumption:spo-risk-bounds} holds (in particular uniformly over $\omega \in \Lambda \times \Theta$), for any $t \in \{1, \ldots, T\}$, we have
        \begin{align*}
            &\bbE_{(x_1, \mu_1), \ldots, (x_{t-1}, \mu_{t-1}) \sim \bbP^{t-1}}[\rspo(g_t; \omega_t) - \rspo(g^\ast; \omega_t)] = \\
            &\bbE_{(x_1, \mu_1), \ldots, (x_{t-1}, \mu_{t-1}) \sim \bbP^{t-1}}\left[\bbE_{(x_t, \mu_t) \sim \bbP} [(r_t - V_t (\lambda_t + \zeta \cdot \theta_t))^T (\bar{w}_t - w_t) ~|~ \calF_{t-1}]\right] = \\
            &\bbE_{(x_1, \mu_1), \ldots, (x_{t}, \mu_{t}) \sim \bbP^{t}}\left[(r_t - V_t (\lambda_t + \zeta \cdot \theta_t))^T (\bar{w}_t - w_t) \right] \leq \kappa_{\mathrm{risk}} \cdot (t-1)^{-\alpha}.
        \end{align*}
        Hence, it holds that
        \begin{equation}\label{eqn:iteration_risk}
            \bbE\left[(r_t - V_t (\lambda_t + \zeta \cdot \theta_t))^T (\bar{w}_t - w_t)\right] \leq \kappa_{\mathrm{risk}} \cdot (t-1)^{-\alpha}
        \end{equation}
        where the expectation above is with respect to all randomness of \Cref{alg:practical}.
        Also, since $g^*(x_t)$ is the Bayes estimator, for any policy $\pi$, it holds that 
        \begin{equation*}
            \bbE_{(x_t, \mu_t) \sim \bbP} [(r_t - V_t (\lambda_t + \zeta \cdot \theta_t))^T (\pi(x_t) - \bar{w}_t) ~|~ \calF_{t-1}] \le 0,
        \end{equation*}
        and
        \begin{equation*}
            \bbE[(r_t - V_t (\lambda_t + \zeta \cdot \theta_t))^T (\pi(x_t) - \bar{w}_t)] \le 0.
        \end{equation*}
        Therefore, combining \eqref{eqn:iteration_risk} with the above yields 
        \begin{equation*}
            \bbE_{(x_t, \mu_t) \sim \bbP} [(r_t - V_t (\lambda_t + \zeta \cdot \theta_t))^T (\pi(x_t) - w_t)] \le \kappa_{\mathrm{risk}} \cdot (t-1)^{-\alpha}. 
        \end{equation*}
        Then by taking the summation over $t = 1, \dots, T$, we have 
        \begin{equation*}
            \mathcal{R}_g(T) \le \kappa_{\mathrm{risk}} \cdot \mathcal{O} (T^{1 - \alpha}). 
        \end{equation*}
    \end{proof}
    
    The following lemma provides an important property of the constant $B_\calV$ in relation to the distance function $d_{\mathcal{V}}(v)$ and its conjugate.
    \begin{lemma}\label{lemma:property-bv}
        For any $v \not\in \mathcal{V}$ and $\kappa \in [0, 1]$, there exists $\theta \in \Theta$ such that 
        \begin{equation*}
            \kappa \cdot d_\calV^*(\theta) - \theta^T v \le (\kappa - 1) \cdot B_\calV. 
        \end{equation*}
    \end{lemma}
    \begin{proof}
        Since $v \not\in \calV$ and $\calV$ is a closed convex set, by the separating hyperplane theorem, there exists a vector $\theta \in \bbR^m$ with $\vvta{\theta}_* = 1$ such that for any $v^\dagger \in \calV$, it holds that $\theta^T v^\dagger < \theta^T v$. 
        Let $\tilde{v}' \in \bbR^m$ be such that $\vvta{\tilde{v}'} = 1$ and $\theta^T \tilde{v}' = 1$. 
        Since $0 \in \mathcal{V}$ and $\sup_{v^\dagger \in \mathcal{V}}\{\theta^T v^\dagger\} < \theta^T v < +\infty$, there exists a constant $\iota \geq 0$ such that $v' \gets \iota \cdot \tilde{v}' \in \partial \calV$.
        Therefore, it holds that $d_\calV^*(\theta) \ge \theta^T v' - d_\calV(v') = \theta^T v' = \vvta{v'} \ge B_\calV$.  
        
        Now consider an arbitrary $\tilde{v}^\circ \in \bbR^m$ and, by the definition of the distance function, let $v^\circ \in \calV$ be such that $\vvta{\tilde{v}^\circ - v^\circ} = d_\calV(\tilde{v}^\circ)$. 
        Then, it holds that 
        \begin{equation*}
            d_\calV^*(\theta) \ge \theta^T v^\circ = \theta^T \tilde{v}^\circ + \theta^T (v^\circ - \tilde{v}^\circ) \ge \theta^T \tilde{v}^\circ - \vvta{v^\circ - \tilde{v}^\circ} = \theta^T \tilde{v}^\circ - d_\calV(\tilde{v}^\circ).
        \end{equation*}
        Since the above is true for arbitrary $\tilde{v}^\circ$, by taking the supremum over the closed set $\mathcal{V}$ and using the fact that $\sup_{v^\dagger \in \mathcal{V}}\{\theta^T v^\dagger\} < \theta^T v < +\infty$, we have that $d_\calV^*(\theta) = \theta^T v^\circ$ for some $v^\circ \in \mathcal{V}$. 
        Then it holds that 
        \begin{align*}
            \kappa \cdot d_\calV^*(\theta) - \theta^Tv & \le \kappa \cdot d_\calV^*(\theta) - \theta^T v^\circ \\ 
            & = (\kappa - 1) \cdot d_\calV^*(\theta) \\ 
            & \le (\kappa - 1) \cdot B_\calV, 
        \end{align*}
        where the last inequality holds since $\kappa - 1 \le 0$. 
    \end{proof}
    
    \subsection{Proof of \Cref{thm:regret-hard}}
    
    The following lemma presents the regret from the suboptimality of the dual variables in the hard constraint case. Recall the definition of the stopping time $\tau := \min \{t \le T: \frac1T \sum_{s=1}^t V_s^T w_s \not\in \mathcal{V}\}$, $\ravg^\tau := \frac1T \sum_{t=1}^\tau r_t^T w_t$ and $\vavg^\tau := \frac{1}{T} \sum_{t=1}^\tau v_t$.
    
    \begin{lemma}\label{lemma:regret-hard}
        For any feasible policy $\pi(\cdot): \mathcal{X} \to \mathcal{S}$ and $T \geq 1$, \Cref{alg:practical} satisfies
        \begin{align*}
            &(A): \ \ \bbE \lrr{\frac1T \cdot \sum_{t=1}^\tau (V_t \theta_t)^T (\pi(x_t) - w_t)} \le \lr{\frac{\tau}{T} - 1} B_{\mathcal{V}} + \frac{\calRup_{\theta}(T)}{T}, \text{ and }\\
            &(B): \ \ \bbE \lrr{\frac1T \cdot \sum_{t=1}^\tau (V_t \lambda_t)^T (\pi(x_t) - w_t)} \le \frac{\tau}{T} \cdot (-u)(\mathrm{con}(\pi)) - \bbE[(-u)(\vavg^\tau)] + \frac{\calRup_\lambda(T)}{T}.  
        \end{align*}
    \end{lemma}

    \begin{proof}
        Let us first prove inequality $(A)$. Since $\pi$ is a feasible policy, the Fenchel-Young inequality yields
        \begin{equation}\label{eq:lemma-regret-hard-1}
            \bbE \lrr{\pi(x_t)^T V_t \theta_t - d_{\mathcal{V}}^*(\theta_t) \vert \mathcal{F}_{t-1}} = (\mathrm{con}(\pi))^T \theta_t - d_{\mathcal{V}}^*(\theta_t) \le d_{\mathcal{V}}(\mathrm{con}(\pi)) = 0. 
        \end{equation}
        Also, \Cref{lemma:regret-mirror} guarantees that for any $\theta \in \Theta$, we have
        \begin{equation*}
            \sum_{t=1}^\tau (\phi_t(\theta_t) - \phi_t(\theta)) \le \calRup_\theta(T). 
        \end{equation*}
        Given the definition of $\phi(\cdot)$, which is $\phi_t(\theta') = -v_t^T \theta' + d_\calV^*(\theta')$, and $v_t = V_t^T w_t$, the above is equivalent to
        \begin{equation}\label{eq:lemma-regret-hard-2}
            \sum_{t=1}^\tau (-w_t^T V_t \theta_t + d_\calV^*(\theta_t) + w_t^T V_t \theta - d_\calV^*(\theta)) \le \calRup_\theta(T)
        \end{equation}
        Therefore, for any $\theta \in \Theta$, it holds that 
        \begin{align*}
            & \bbE \lrr{\frac1T \cdot \sum_{t=1}^\tau (V_t \theta_t)^T (\pi(x_t) - w_t)} \\ 
            \le \, & \bbE \lrr{\frac1T \cdot \sum_{t=1}^\tau \lr{\pi(x_t)^T V_t \theta_t - w_t^T V_t \theta + d_{\mathcal{V}}^*(\theta) - d_{\mathcal{V}}^*(\theta_t)}} + \frac{\calRup_{\theta}(T)}{T} \\ 
            \le \, & \bbE \lrr{\frac1T \cdot \sum_{t=1}^\tau \lr{- w_t^T V_t \theta + d_{\mathcal{V}}^*(\theta)}} + \frac{\calRup_{\theta}(T)}{T} \\ 
            = \, & \bbE \lrr{\frac{\tau}{T} \cdot d_{\mathcal{V}}^*(\theta) - \theta^T \vavg^\tau} + \frac{\calRup_{\theta}(T)}{T}. 
        \end{align*}
        where the first inequality comes from \eqref{eq:lemma-regret-hard-2}, and the second inequality comes from \eqref{eq:lemma-regret-hard-1}.
        
        If $\tau = T$, we pick $\theta \gets 0$, and it holds that $\frac{\tau}{T} \cdot d_{\mathcal{V}}^*(\theta) = \theta^T \vavg^\tau = 0$, which implies $(A)$.
        If $\tau < T$, it implies $\vavg^\tau \not\in \mathcal{V}$. 
        Then following the results in \Cref{lemma:property-bv} we know that there exists $\theta \in \Theta$ such that $\frac{\tau}{T} \cdot d_{\mathcal{V}}^*(\theta) - \theta^T \vavg^\tau \le \lr{\frac{\tau}{T} - 1} \cdot B_{\mathcal{V}}$. 
        Therefore, for both $\tau = T$ and $\tau < T$, we have $(A)$. 
        
        Let us first prove inequality $(B)$. First, the Fenchel-Young inequality again yields
        \begin{equation}\label{eq:lemma-regret-hard-3}\begin{aligned}
            \frac1T \sum_{t=1}^\tau \bbE \lrr{\pi(x_t)^T V_t \lambda_t - (-u)^*(\lambda_t) \vert \mathcal{F}_{t-1}} 
            & = \frac1T \sum_{t=1}^\tau \lr{\mathrm{con}(\pi)^T \lambda_t - (-u)^* (\lambda_t)} \\ 
            & \le \frac1T \sum_{t=1}^\tau (-u)(\mathrm{con}(\pi)) = \frac{\tau}{T} (-u)(\mathrm{con}(\pi)). 
        \end{aligned}\end{equation}
        Recall the definition of $\phi(\cdot)$, which is $\phi(\lambda) = -w_t^T V_t \lambda + (-u)^*(\lambda)$. Then, following \Cref{lemma:regret-mirror}, for any $\lambda \in \Lambda$, it holds that 
        \begin{equation}\label{eq:lemma-regret-hard-4}
            \sum_{t=1}^\tau \lr{w_t^T V_t \lambda - (-u)^*(\lambda) - w_t^T V_t \lambda_t + (-u)^*(\lambda_t)} = \sum_{t=1}^\tau (\xi_t(\lambda_t) - \xi_t(\lambda)) \le \calRup_{\lambda}(T). 
        \end{equation}
        Let $v' \gets \frac{T}{\tau} \cdot \vavg^\tau$, and pick $\lambda \in \Lambda$ such that $(-u)(v') = (v')^T \lambda - (-u)^*(\lambda)$. 
        Then, using concavity of $u(\cdot)$ and $u(0) = 0$, it holds that 
        \begin{equation}\label{eq:lemma-regret-hard-5}\begin{aligned}
            (-u)(\vavg^\tau) & = (-u) \lr{\frac{\tau}{T} \cdot v' + \lr{1 - \frac{\tau}{T}} \cdot 0} \\ 
            & \le \frac{\tau}{T} \cdot (-u)(v') + \lr{1 - \frac{\tau}{T}} \cdot (-u)(0) \\ 
            & = \frac{\tau}{T} \cdot (\lambda^T v' - (-u)^*(\lambda)) \\ 
            & = \frac1T \cdot \sum_{t=1}^\tau (w_t^T V_t \lambda - (-u)^*(\lambda)). 
        \end{aligned}\end{equation}
        By adding \eqref{eq:lemma-regret-hard-3}, \eqref{eq:lemma-regret-hard-4} and \eqref{eq:lemma-regret-hard-5}, we arrive at the inequality $(B)$. 

    \end{proof}
    
    \begin{proof}[Proof of \Cref{thm:regret-hard}]
        By combining the results in \Cref{lemma:regret-prediction} and \Cref{lemma:regret-hard}, for any feasible policy $\pi$, it holds that 
        \begin{align*}
            & \frac{\tau}{T} \cdot (\mathrm{rew}(\pi) + u(\mathrm{con}(\pi))) + \lr{1 - \frac{\tau}{T}} \cdot B_{\mathcal{V}} \zeta - \bbE \lrr{\frac1T \sum_{t=1}^\tau r_t^T w_t + u(\vavg^\tau)} \\ 
            \le \, & \frac1T \cdot (\mathcal{R}_g(T) + \calRup_\lambda(T) + \zeta \cdot \calRup_\theta(T)). 
        \end{align*}
        Also, it holds that $\mathcal{R}_g(T) \le \kappa_{\mathrm{risk}} \cdot \mathcal{O}(T^{1-\alpha})$, $\calRup_\lambda(T) \le D_v \sqrt{D_\Lambda} \cdot \mathcal{O}(T^{1/2})$, and $\calRup_\theta(T) \le D_v \sqrt{D_\Theta} \cdot \mathcal{O}(T^{1/2})$. 
        Therefore, when $\zeta \ge \frac{\mathrm{OPT}}{B_{\mathcal{V}}}$, it holds that 
        \begin{equation*}
            \mathrm{OPT} - \bbE [\ravg^\tau + u(\vavg^\tau)] \le \kappa_{\mathrm{MD}} \cdot \mathcal{O}(T^{-1/2}) + \kappa_{\mathrm{risk}} \cdot \mathcal{O}(T^{- \alpha}). 
        \end{equation*}
    \end{proof}
    
    \subsection{Proof of \Cref{thm:regret-soft}}
    
    The following lemma presents the regret from the suboptimality of the dual variables in the soft constraint case. 
    
    \begin{lemma}\label{lemma:regret-soft}
        For any feasible policy $\pi(\cdot): \mathcal{X} \to \mathcal{S}$ and $T \geq 1$, \Cref{alg:practical} satisfies
        \begin{align*}
            &(A): \ \ \bbE \lrr{\frac1T \cdot \sum_{t=1}^T (V_t \lambda_t)^T (\pi(x_t) - w_t)} \le (-u)(\mathrm{con}(\pi)) - \bbE [(-u)(\vavg)] + \frac{\calRup_{\lambda}(T)}{T}, \\
            &(B): \ \ \bbE \lrr{\frac1T \cdot \sum_{t=1}^T (V_t \theta_t)^T (\pi(x_t) - w_t)} \le - \bbE [d_{\mathcal{V}}(\vavg)] + \frac{\calRup_{\theta}(T)}{T}. 
        \end{align*}
    \end{lemma}
    \begin{proof}
        First, there exists $\lambda \in \Lambda$ such that 
        \begin{equation}\label{eq:lemma-regret-soft-1}
            (-u)(\vavg) = \vavg^T \lambda - (-u)^*(\lambda) = \frac1T \sum_{t=1}^T \lr{w_t^T V_t \lambda - (-u)^*(\lambda)}. 
        \end{equation}
        Next, the Fenchel-Young inequality yields
        \begin{equation}\label{eq:lemma-regret-soft-2}\begin{aligned}
            \frac1T \sum_{t=1}^T \bbE \lrr{\pi(x_t)^T V_t \lambda_t - (-u)^*(\lambda_t) \vert \mathcal{F}_{t-1}} 
            & = \frac1T \sum_{t=1}^T \lr{\mathrm{con}(\pi)^T \lambda_t - (-u)^* (\lambda_t)} \\ 
            & \le \frac1T \sum_{t=1}^T (-u)(\mathrm{con}(\pi)) = (-u)(\mathrm{con}(\pi)). 
        \end{aligned}\end{equation}
        \Cref{lemma:regret-mirror} guarantees that 
        \begin{equation*}
            \sum_{t=1}^T (\xi_t(\lambda_t) - \xi_t(\lambda)) \le \calRup_\lambda(T). 
        \end{equation*}
        Given the definition of $\xi(\cdot)$, which is $\xi_t(\lambda') = -v_t^T \lambda' + (-u)^*(\lambda')$, and $v_t = V_t^T w_t$, it holds that 
        \begin{equation}\label{eq:lemma-regret-soft-3}
            \sum_{t=1}^T \lr{-w_t^T V_t \lambda_t + (-u)^*(\lambda_t) + w_t^T V_t \lambda - (-u)^*(\lambda)} \le \calRup_\lambda(T). 
        \end{equation}
        Therefore, by adding \eqref{eq:lemma-regret-soft-1}, \eqref{eq:lemma-regret-soft-2}, and \eqref{eq:lemma-regret-soft-3}, it holds that 
        \begin{align*}
            & \bbE \lrr{\frac1T \cdot \sum_{t=1}^T (V_t \lambda_t)^T (\pi(x_t) - w_t)} \\ 
            \le \, & \bbE \lrr{\frac1T \cdot \sum_{t=1}^T \lrr{(V_t \lambda_t)^T \pi(x_t) - w_t^T V_t \lambda + (-u)^*(\lambda) - (-u)^*(\lambda_t))}} + \frac{\calRup_{\lambda}(T)}{T} \\ 
            \le \, & (-u)(\mathrm{con}(\pi)) - \bbE [(-u)(\vavg)] + \frac{\calRup_{\lambda}(T)}{T}. 
        \end{align*}
        Applying the same reasoning to the other set of dual variables and using that $\mathrm{con}(\pi) \in \mathcal{V}$, we have 
        \begin{equation*}
            \bbE \lrr{\frac1T \cdot \sum_{t=1}^T (V_t \theta_t)^T (\pi(x_t) - w_t)} \le d_{\mathcal{V}}(\mathrm{con}(\pi)) - \bbE [d_{\mathcal{V}}(\vavg)] + \frac{\calRup_{\theta}(T)}{T} = - \bbE [d_{\mathcal{V}}(\vavg)] + \frac{\calRup_{\theta}(T)}{T}. 
        \end{equation*}
    \end{proof}
    
    \begin{proof}[Proof of \Cref{thm:regret-soft}]
        By combining the results in \Cref{lemma:regret-prediction} and \Cref{lemma:regret-soft}, for any feasible policy $\pi$, it holds that 
        \begin{align*}
            & \mathrm{rew}(\pi) + u(\mathrm{con}(\pi)) - \bbE \lrr{\frac1T \sum_{t=1}^T r_t^T w_t + u(\vavg)} \\ 
            \le \, & \frac1T \cdot (\mathcal{R}_g(T) + \calRup_\lambda(T) + \zeta \cdot \calRup_\theta(T) - \zeta \cdot \bbE [d_\mathcal{V}(\vavg)]. 
        \end{align*}
        Since $\bbE [d_\mathcal{V}(\vavg)] \ge 0$, it holds that 
        \begin{equation*}
            \mathrm{OPT} - \bbE [\ravg + u(\vavg)] \le \kappa_{\mathrm{MD}} \cdot \mathcal{O}(T^{-1/2}) + \kappa_{\mathrm{risk}} \cdot \mathcal{O}(T^{- \alpha}). 
        \end{equation*}
        Also, when \Cref{assumption:opt-relax} holds, it holds that 
        \begin{equation*}
            \bbE [\ravg + u(\vavg)] - \mathrm{OPT} \le \zeta_{\mathrm{OPT}} \cdot \bbE [d_\mathcal{V}(\vavg)], 
        \end{equation*}
        and therefore, it holds that 
        \begin{equation*}
            (\zeta - \zeta_{\mathrm{OPT}}) \cdot \bbE [d_\mathcal{V}(\vavg)] \le \frac1T \cdot (\mathcal{R}_g(T) + \calRup_\lambda(T) + \zeta \cdot \calRup_\theta(T)). 
        \end{equation*}
        If additionally $\zeta$ satisfies $\zeta \ge 2(\zeta_{\mathrm{OPT}} + \frac{\sqrt{D_\Lambda}}{\sqrt{D_\Theta}} + \frac{\kappa_{\mathrm{risk}}}{D_v \sqrt{D_\Theta}})$, it holds that 
        \begin{equation*}
            \bbE [d_\mathcal{V}(\vavg)] \le D_v \sqrt{D_\Theta} \cdot \mathcal{O} (T^{-1/2}) + \mathcal{O} (T^{-\alpha}). 
        \end{equation*}
    \end{proof}
    
    \subsection{A sufficient condition for \Cref{assumption:opt-relax}}
    Herein we provide a sufficient condition which guarantees the existence of the constant $\zeta_{\mathrm{OPT}}$ in \Cref{assumption:opt-relax}.
    \begin{lemma}\label{lemma:opt-relax}
        Suppose that $\mathrm{OPT}$ is finite and that there exists a policy $\pi^\circ$ such that $\mathrm{con}(\pi^\circ) \in \tn{int}(\calV)$. Then, there exists a constant $\zeta'$ such that 
        \begin{equation*}
            \mathrm{OPT} = \sup_{\pi} \{\mathrm{rew}(\pi) + u(\mathrm{con}(\pi)) - \zeta' \cdot d_\mathcal{V}(\mathrm{con}(\pi))\}. 
        \end{equation*}
    \end{lemma}
    \begin{proof}
        Let $u^\circ \gets \mathrm{con}(\pi^\circ)$. Since $u^\circ \in \tn{int}(\calV)$, there exists a constant $\epsilon > 0$ such that for all $u'$ satisfying $ \vvta{u' - u^\circ} \le \epsilon$, it holds that $u' \in \calV$. 
        Let $\partial \calV$ denote the boundary of the set $\calV$, and define $\calV^{-\epsilon} \gets \{v \in \calV: d_{\partial \calV}(v) \ge \epsilon\}$. 
        Consider 
        \begin{equation*}
            \mathrm{OPT}^{-\epsilon} = \sup_{\pi: \mathrm{con}(\pi) \in \calV^{-\epsilon}} \{\mathrm{rew}(\pi) + u(\mathrm{con}(\pi))\}. 
        \end{equation*}
        Since $\mathrm{con}(\pi^\circ) \in \calV^{-\epsilon}$, we know that $\mathrm{OPT}^{-\epsilon}$ is real-valued. 
        Pick a policy $\pi^\dagger$ such that $\mathrm{con}(\pi^\dagger) \in \calV^{-\epsilon}$ and $\mathrm{rew}(\pi^\dagger) + u(\mathrm{con}(\pi^\dagger)) \ge \mathrm{OPT}^{-\epsilon} - \epsilon$. 
        Let $v^\dagger \gets \mathrm{con}(\pi^\dagger)$. 
        Now for any $\pi \not\in \calV$, let $v \gets \mathrm{con}(\pi)$. Pick $\tilde{v} \in \partial \calV$ such that $\vvta{v - \tilde{v}} = d_{\calV}(v)$ and let $\kappa = d_{\calV}(v) / d_{\partial \calV}(v^\dagger) \le d_{\calV}(v) / \epsilon$. 
        Let $\tilde{v}^\dagger \gets v^\dagger + (v - \tilde{v}) / \kappa$, since $\vvta{\tilde{v}^\dagger - v^\dagger} = d_{\partial \calV}(v^\dagger)$, it holds that $\tilde{v}^\dagger \in \calV$. 
        Also, let $v' \gets \frac{1}{\kappa + 1} \cdot (\kappa \cdot \tilde{u}^\dagger + \tilde{v}) \in \calV$, and it holds that $v' = \frac{1}{\kappa + 1} \cdot (\kappa \cdot u^\dagger + u)$.  
        
        Now define a new policy $\pi'$ by $\pi'(x) := \frac{1}{\kappa + 1} \cdot (\kappa \cdot \pi^\dagger(x) + \pi(x))$ for any $x \in \calX$. 
        The policy $\pi'$ is well-defined since $\mathcal{S}$ is convex and therefore $\pi'(x) \in \mathcal{S}$ for any $x \in \calX$. 
        Also, the policy $\pi'$ is feasible since $\mathrm{con}(\pi') = v' \in \calV$, and therefore, it holds that $\mathrm{OPT} \ge \mathrm{rew}(\pi') + u(\mathrm{con})(\pi')$. 
        On the other hand, since $u(\cdot)$ is concave, it holds that 
        \begin{equation*}
            \mathrm{rew}(\pi') + u(\mathrm{con}(\pi')) \ge \frac{1}{\kappa + 1} \cdot (\kappa \cdot [\mathrm{rew}(\pi^\dagger) + u(\mathrm{con}(\pi^\dagger))] + [\mathrm{rew}(\pi) + u(\mathrm{con}(\pi))]). 
        \end{equation*}
        Therefore, it holds that 
        \begin{align*}
            \mathrm{rew}(\pi) + u(\mathrm{con}(\pi)) - \mathrm{OPT} & \le \kappa \cdot (\mathrm{OPT} - (\mathrm{rew}(\pi^\dagger) + u(\mathrm{con}(\pi^\dagger))) \\ 
            & = \frac{d_{\calV}(\mathrm{con}(\pi))}{d_{\partial \calV}(\mathrm{con}(\pi^\circ))} \cdot (\mathrm{OPT} - (\mathrm{OPT}^{-\epsilon} - \epsilon)) \\ 
            & \le d_{\calV}(\mathrm{con}(\pi)) \cdot \lr{\frac{\mathrm{OPT} - \mathrm{OPT}^{-\epsilon}}{\epsilon} + 1}. 
        \end{align*}
        By setting $\zeta' \gets 1 + (\mathrm{OPT} - \mathrm{OPT}^{-\epsilon}) / \epsilon$, for any $\pi \not\in \calV$, it holds that 
        \begin{equation*}
            \mathrm{OPT} \ge \mathrm{rew}(\pi) + u(\mathrm{con}(\pi)) - \zeta' \cdot d_{\calV}(\mathrm{con}(\pi)), 
        \end{equation*}
        and we conclude the proof. 
    \end{proof}
    Now given the results in \Cref{lemma:opt-relax}, for any $\epsilon > 0$, it holds that 
    \begin{align*}
        \mathrm{OPT}^\epsilon &= \sup_{\pi: d_\calV(\mathrm{con}(\pi)) \le \epsilon} \{\mathrm{rew}(\pi) + u(\mathrm{con(\pi)})\} \\ 
        & \le \sup_{\pi} \, \{\mathrm{rew}(\pi) + u(\mathrm{con(\pi)}) - \zeta' \cdot [d_\calV(\mathrm{con}(\pi)) - \epsilon])\} \\ 
        & = \sup_{\pi} \, \{\mathrm{rew}(\pi) + u(\mathrm{con(\pi)}) - \zeta' \cdot d_\calV(\mathrm{con}(\pi)))\} + \zeta' \cdot \epsilon \\ 
        & \le \mathrm{OPT} + \zeta' \cdot \epsilon, 
    \end{align*}
    and therefore we show the existence of $\zeta_{\mathrm{OPT}}$.

\section{Trade-off between regret and computation cost}\label{appendix:beta-efficient}
    In each iteration of \Cref{alg:practical}, one essential step is to update the prediction model based on all the previous observations and the current dual variables. 
    Although it may be possible to perform this update efficiently -- for example, one could use a warm-starting procedure depending on the structure of the hypothesis class -- the decision-maker may still not want to update the prediction model at each iteration, especially if decisions need to be made quickly. 
    To address this issue, we develop a more computationally efficient version of our algorithm, which only updates the prediction model at a sublinear rate, and is formally described as follows.
    \begin{definition}
        For any constant $\beta \ge 1$, the $\beta$-efficient version of \Cref{alg:practical} is an algorithm which is same as \Cref{alg:practical} but only updates the dual variables and prediction model at iteration $t = \lfloor k^\beta \rfloor$ for all positive integer $k$. 
    \end{definition}
    From the prediction model update frequency of a $\beta$-efficient version of \Cref{alg:practical}, we notice that a total number of $T^{1/\beta}$ prediction model updates is required. 
    We provide the regret analysis of a $\beta$-efficient version of \Cref{alg:practical} in \Cref{thm:regret-beta}. 
    \begin{theorem}\label{thm:regret-beta}
        In the hard constraints case, suppose that the assumptions of \Cref{thm:regret-hard} hold and consider the $\beta$-efficient version of \Cref{alg:practical} for some constant $\beta \in (0, 1]$. Then we have the following guarantee: 
        \begin{equation*}
            \mathrm{OPT} - \bbE [\ravg^\tau + u(\vavg^\tau)] \le \kappa_{\mathrm{MD}} \cdot \mathcal{O} (T^{-1/2\beta}) + \kappa_{\mathrm{risk}} \cdot \mathcal{O} (T^{- \alpha}). 
        \end{equation*}
        In the soft constraints case, suppose that the assumptions of \Cref{thm:regret-soft} hold and consider the $\beta$-efficient version of \Cref{alg:practical} for some constant $\beta \in (0, 1]$. Then we have the following guarantees:
        \begin{equation*}
            \mathrm{OPT} - \bbE [\ravg + u(\vavg)] \le \kappa_{\mathrm{MD}} \cdot \mathcal{O} (T^{-1/2\beta}) + \kappa_{\mathrm{risk}} \cdot \mathcal{O} (T^{- \alpha}), 
        \end{equation*}
        and if additionally the budget penalty parameter $\zeta$ satisfies $\zeta \ge 2(\zeta_{\mathrm{OPT}} + \frac{\sqrt{D_\Lambda}}{\sqrt{D_\Theta}} + \frac{\kappa_{\mathrm{risk}}}{D_v \sqrt{D_\Theta}})$, it holds that 
        \begin{equation*}
            \bbE [d_\mathcal{V}(\vavg)] \le D_v \sqrt{D_\Theta} \cdot \mathcal{O} (T^{-1/2\beta}) + \mathcal{O} (T^{-\alpha}). 
        \end{equation*}
    \end{theorem}
    \begin{proof}
        When the updating sequence is $t = t_1, \dots, t_K$, the regret from online mirror descent can be bounded by
        \begin{equation*}
            \sum_{t=1}^T (\xi_t(\lambda) - \xi_t(\lambda_t)) \le \frac{D_\Lambda}{2 \eta_\lambda} + \sum_{k=1}^K \frac{\eta_\lambda}{2} G_\lambda^2 (t_k-t_{k-1})^2. 
        \end{equation*}
        Also, it holds that 
        \begin{equation*}
            \sum_{k=1}^K (t_k - t_{k-1})^2 = \sum_{k=1}^K (\beta k^{\beta - 1})^2 = \frac{\beta^2}{2 \beta - 1} \cdot K^{2 \beta - 1} = \frac{\beta^2}{2 \beta - 1} \cdot T^{2 - 1 / \beta}. 
        \end{equation*}
        and therefore by setting $\eta_\lambda \gets \frac{\sqrt{D_\Lambda}}{G_\Lambda T^{1 - 1/2\beta}}$, we have
        \begin{equation*}
            \mathcal{R}_\lambda(T) = \sum_{t=1}^T (\xi_t(\lambda_t) - \xi_t(\lambda)) \le G_\Lambda \sqrt{D_\Lambda} \cdot \mathcal{O}(T^{1 - 1/2\beta}). 
        \end{equation*}
        For the same reason we also have 
        \begin{equation*}
            \mathcal{R}_\theta(T) = \sum_{t=1}^T (\phi_t(\theta_t) - \phi_t(\theta)) \le G_\Theta \sqrt{D_\Theta} \cdot \mathcal{O}(T^{1 - 1/2\beta}). 
        \end{equation*}
        Now using the proof in \Cref{thm:regret-hard} and \Cref{thm:regret-soft} again we can get the results in \Cref{thm:regret-beta}. 
    \end{proof}
    From the regret analysis, we see that the idea of $\beta$-efficient version of our algorithm is beneficial when the learning of the prediction model has a slower rate, \ie when $\alpha < \frac12$. 
    In this case, we can set $\beta \gets 1/2\alpha$, and the $\beta$-version of \Cref{alg:practical} will have a same regret order as the original algorithm, while maintaining a sublinear total number of prediction model updates.

\section{Experimental details}\label{appendix:experiment}

    In this section, we discuss the details in our numerical experiments. We first provide the results from another experiments on longest path instances, and then discuss issues concerning computational resources and detailed data generation processes. 
    
    \subsection{Additional experiments}
    In this section, we consider a longest path problem on a $4 \times 4$ directed grid network with edges pointing north and east, and the goal is to go from the southwest corner to the northeast corner while maximizing the rewards collected along each edge.
    In this case, the feasible region $\mathcal{S}$ can be modeled as the convex hull of all possible routes. We assume there is no learning in the consumption matrix, \ie the consumption is just the decision itself, namely $V_t = I_d$, i.e., $v_t = w_t$. Also, we would like to not utilize any edge too frequently and model this idea by setting the resource consumption feasible set as $\mathcal{V} = \{v: v \le 0.6 \cdot e\}$, which is a soft constraint, and letting the utility function be $u(v) = \sum_i v_i (1 - v_i)$. 
    The data generation process is similar as the one in the multi-dimensional knapsack instances, and we set the total number of arrivals to 1000 in this experiment. 
    \Cref{fig:shortest_path_deg} displays the empirical performance of each method. We observe that the SPO+ loss which accounts for both dual variables and the decision feasible region $\mathcal{S}$ dominates all cases, and it is more beneficial when the polynomial degree is higher. 
    \begin{figure}
        \centering
        \includegraphics[width=\textwidth]{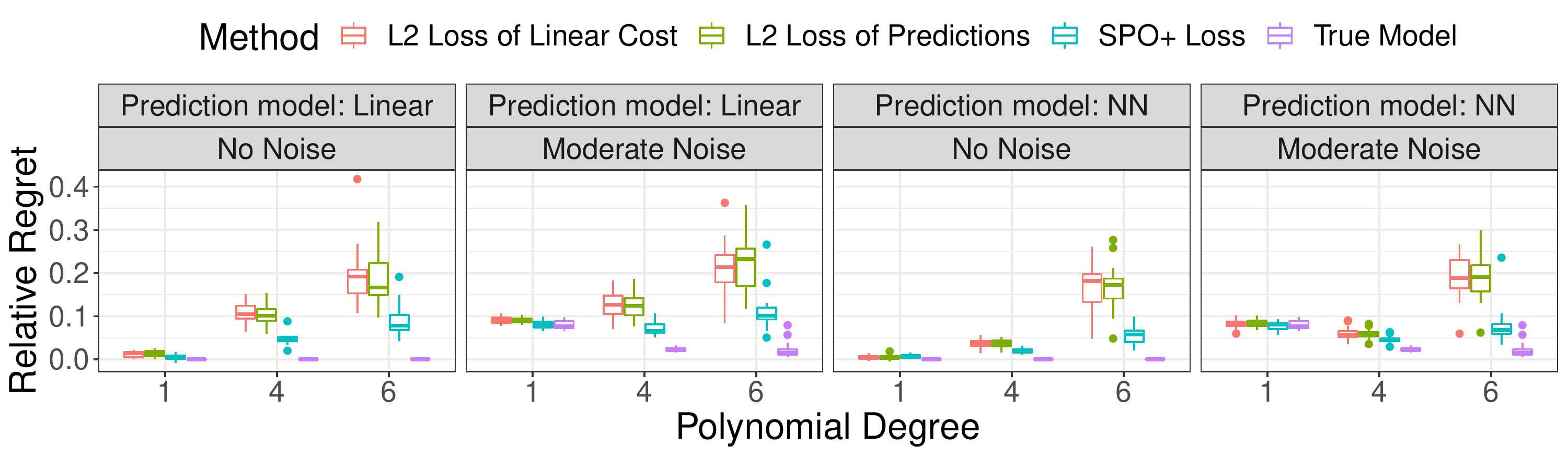}
        \includegraphics[width=\textwidth]{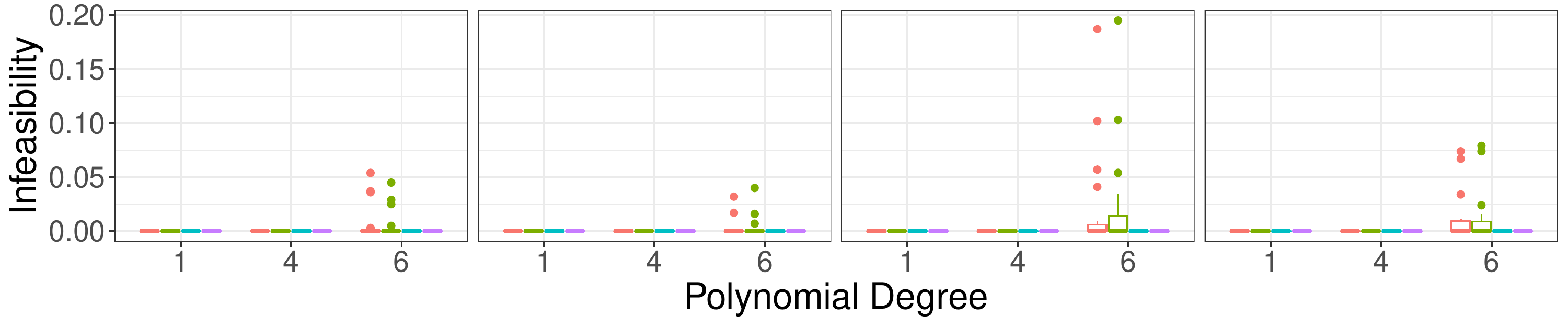}
        \caption{Relative regret and infeasibility for different loss functions on shortest path instances.}
        \label{fig:shortest_path_deg}
    \end{figure}
    
    \subsection{Experiment settings and data generation processes}
    For all loss functions, we use the Adam method of \cite{kingma2015adam} to train the weight matrices and bias coefficients in the prediction models, and we update the dual variables and prediction models every 10 iterations. For each instance, \eg each value of the total time horizon and the polynomial degree, we run $40$ independent trials on one core of Intel Xeon Skylake 6230 @ 2.1 GHz.
    
    \paragraph{Multi-dimensional knapsack instances.} 
    Let us first describe the detailed data generation process in the multi-dimensional knapsack instances. In this experiment, we set the dimension of the feature vector $p = 5$, the dimension of decision vector $d = 10$, and the dimension of the resource vector $m = 3$. We first generate the weight matrix $W \in \bbR^{d(m+1) \times p}$, whereby each entry of $W$ is a Bernoulli random variable with the probability $\bbP(B_{jk} = 1) = \frac12$. We then generate the arrivals $\{(x_i, r_i, V_i)\}_{i=1}^p$ independently by the following procedure: 
    \begin{enumerate}
        \item Generate the feature vector $x$ from a standard multivariate normal distribution, namely $x \sim \mathcal{N}(0, I_p)$. 
        \item Generate the vectorization of the reward vector $r$ and the resource consumption matrix $V$ according to 
        \begin{equation*}
            \tn{vec}(r, V)_j \gets \lrr{1 + \lr{1 + \frac{W_j^T x}{\sqrt{p}}}^{\mathrm{deg}}} \epsilon_j, 
        \end{equation*} 
        for $j = 1, \dots, d(m+1)$, where $W_j$ is the $j$-th row of matrix $W$. Here $\mathrm{deg}$ is the fixed degree parameter and $\epsilon_j$, the multiplicative noise term, is a random variable which independently generated from the uniform distribution $[1 - \bar{\epsilon}, 1 + \bar{\epsilon}]$ for a fixed noise half width $\bar{\epsilon} \ge 0$. 
        In particular, $\bar{\epsilon}$ is set to $0$ for ``no noise'' instances and $0.5$ for ``moderate noise'' instances. 
    \end{enumerate}
    
    \paragraph{Longest path instances.}
    Let us finally describe the detailed data generation process in the longest path instances. In this experiment, we set the dimension of the feature vector $p = 5$. Also, since the graph is a $4 \times 4$ grid, the dimension of both decision and resource consumption vector will be $d = m = 24$, which is the number of edges in the graph. Since the resource consumption matrix is always the identity matrix, we only need to generate the reward vector based on the feature. Therefore, the weight matrix is $W \in \bbR^{d \times p}$. The remaining part is the same as the data generation process in the multi-dimensional knapsack instances. 

\end{document}